\documentclass[sigconf]{acmart}
\usepackage{bm}
\usepackage{amsfonts}
\usepackage{graphicx}
\usepackage{textcomp}
\usepackage{xcolor}
\usepackage{url}
\usepackage{threeparttable} 
\usepackage[noend]{algpseudocode}
\usepackage{algorithm}
\algnewcommand\algorithmicinput{\textbf{Input:}}
\algnewcommand\Input{\item[\algorithmicinput]}
\algnewcommand\algorithmicoutput{\textbf{Output:}}
\algnewcommand\Output{\item[\algorithmicoutput]}
\usepackage{array}
\usepackage{booktabs}
\usepackage{amsthm}
\usepackage{multirow}
\usepackage{graphicx}
\usepackage{caption}
\usepackage{subcaption}
\usepackage{booktabs}
\usepackage{amsmath}
\usepackage{ulem}
\usepackage{pgf}
\usepackage{pgffor}
\usepackage{hyperref}
\usepackage{setspace}
\usepackage[skip=0pt]{caption} 
\usepackage{etoolbox}
\usepackage{enumitem}

\usepackage{float}
\usepackage{enumitem}
\usepackage{lipsum}   
\usepackage{pifont}   

\setlist{nosep, leftmargin=*}
\newtheorem{definition}{\textbf{Definition}}

\newtheorem{example}{\textbf{Example}}
\newtheorem{theorem}{\textbf{Theorem}}

\newcommand{\ModelName}{NS-Mem}
\newcommand{\GenerationName}{SK-Gen}

\makeatletter
\patchcmd{\proof}
  {\topsep}
  {3pt \@plus 0.1ex \@minus 0.1ex}
  {}{}
\patchcmd{\endproof}
  {\topsep}
  {3pt \@plus 0.1ex \@minus 0.1ex}
  {}{}
\makeatother

\setcopyright{none}
\begin{document}

    \title{Advancing Multimodal Agent Reasoning with Long-Term Neuro-Symbolic Memory}
    \author{Rongjie Jiang}
    \affiliation{
    \institution{University of New South Wales} \city{Sydney}
    \country{Australia}
    }
    \email{rongjie.jiang@student.unsw.edu.au}

    \author{Jianwei Wang\textsuperscript{*}}
    \affiliation{
    \institution{University of New South Wales} \city{Sydney}
    \country{Australia}
    }
    \email{jianwei.wang1@unsw.edu.au}

    \author{Gengda Zhao}
    \affiliation{
    \institution{University of New South Wales} \city{Sydney}
    \country{Australia}
    }
    \email{gengda.zhao@unsw.edu.au}

    \author{Chengyang Luo}
    \affiliation{
    \institution{Zhejiang University} \city{Hangzhou}
    \country{China}
    }
    \email{luocy1017@zju.edu.cn}

    \author{Kai Wang}
    \affiliation{
    \institution{Shanghai Jiao Tong University} \city{Shanghai}
    \country{China}
    }
    \email{w.kai@sjtu.edu.cn}

    \author{Wenjie Zhang}
    \affiliation{
    \institution{University of New South Wales} \city{Sydney}
    \country{Australia}
    }
    \email{wenjie.zhang@unsw.edu.au}

    \renewcommand{\shortauthors}{Jiang et al.}

\begin{abstract}
Recent advances in large language models have driven the emergence of intelligent agents operating in open-world, multimodal environments.
To support long-term reasoning, such agents are typically equipped with external memory systems. However, most existing multimodal agent memories rely primarily on neural representations and vector-based retrieval, which are well-suited for inductive, intuitive reasoning but fundamentally limited in supporting analytical, deductive reasoning critical for real-world decision making.
To address this limitation, we propose \textbf{\ModelName{}}, a long-term neuro-symbolic memory framework designed to advance multimodal agent reasoning by integrating neural memory with explicit symbolic structures and rules. Specifically, \ModelName{} is operated around three core components of a memory system: (1) a three-layer memory architecture that consists episodic layer, semantic layer and logic rule layer, (2) a memory construction and maintenance mechanism implemented by \GenerationName{} that automatically consolidates structured knowledge from accumulated multimodal experiences and incrementally updates both neural representations and symbolic rules, and (3) a hybrid memory retrieval mechanism that combines similarity-based search with deterministic symbolic query functions to support structured reasoning. Experiments on real-world multimodal reasoning benchmarks demonstrate that Neural-Symbolic Memory achieves an average 4.35\% improvement in overall reasoning accuracy over pure neural memory systems, with gains of up to 12.5\% on constrained reasoning queries, validating the effectiveness of \ModelName{}.
\end{abstract}
    \keywords{Neural-Symbolic AI, Memory Systems, Knowledge Representation, Open-World Agents, Structured Reasoning}

    \maketitle
    \newcommand\codeavailabilityurl{https://anonymous.4open.science/r/NSTF-842F}
\label{code}
\ifdefempty{\codeavailabilityurl}{}{
\begingroup\small\noindent\raggedright\textbf{Code Availability:}\\
The source code of this paper has been made publicly available at \url{\codeavailabilityurl}.
\endgroup
}
    \section{Introduction}
    \label{sec.intro}

\begin{figure}[t]
\centering
\includegraphics[width=0.95\columnwidth]{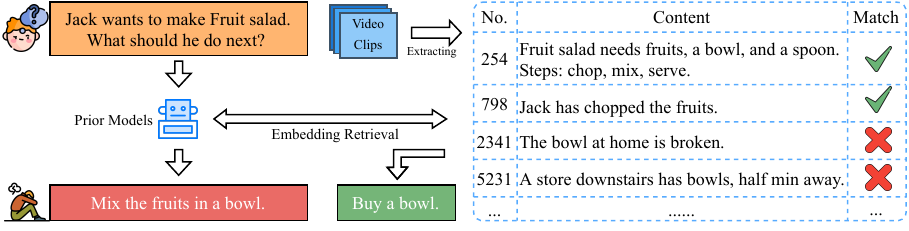}
\caption{An example of a vector-centric multimodal agent on a constrained query.}
\label{fig:example}
\end{figure}

The rapid advancement of Large Language Models (LLMs) has fundamentally redefined the landscape of intelligent agents, empowering them as sophisticated systems that perceive their environment, reason over observations, and execute actions to achieve complex goals~\cite{russell2010artificial, wooldridge1995intelligent}. In real-world scenarios, environments are inherently multimodal, spanning modalities such as textual and visual data. Consequently, developing multimodal agents has become a primary frontier in achieving general-purpose autonomy~\cite{shridhar2020alfred, tang2019coin}. However, the successful deployment of these agents needs to continuously accumulate knowledge from streams of heterogeneous observations, organize this information effectively, and retrieve contextually relevant insights to support flexible decision-making under varying constraints. 
At the core of this capability lies the memory module, the essential substrate for transforming raw multimodal streams into persistent, actionable knowledge.

Recent advances in memory-augmented multimodal agents have made significant progress. 
MemGPT~\cite{packer2023memgpt} introduces explicit memory management for LLM-based agents, enabling them to handle information beyond context window limitations. MovieChat~\cite{song2024moviechat} and MA-LMM~\cite{he2024malmm} construct persistent memory structures for long-form video understanding, while VideoAgent~\cite{wang2024videoagent} employs iterative frame selection guided by memory. M3-Agent~\cite{m3agent} further advances this line by introducing VideoGraph, which organizes memories into episodic nodes representing specific events and semantic nodes capturing abstracted concepts, drawing inspiration from human cognitive systems~\cite{tulving1972episodic}. These approaches typically employ retrieval-augmented generation (RAG)~\cite{lewis2020retrieval} with vector-based similarity search, achieving strong performance on factual recall and semantic matching tasks.


\noindent\textbf{Motivations}. 
Despite their progress, most multimodal agent memories are vector-centric, relying primarily on neural embeddings for storage and retrieval and sometimes being augmented with lightweight relational structures~\cite{DBLP:journals/fcsc/WangMFZYZCTCLZWW24}.
Such designs are well-suited for System 1 style reasoning~\cite{kahneman2011thinking}, enabling associative inference through semantic similarity, including inductive reasoning that generalizes from past experiences, analogical reasoning that transfers knowledge across related contexts~\cite{DBLP:journals/corr/abs-2212-09196}, and associative recall that links semantically proximate concepts. These capabilities are effective for fuzzy matching and commonsense retrieval in open-ended environments. However, they remain fundamentally limited in supporting System 2 style reasoning that is critical for real-world decision making under explicit constraints~\cite{DBLP:journals/corr/abs-2305-15771, DBLP:conf/iclr/Saparov023}. This includes deductive reasoning over dependencies such as understanding prerequisites and ordering, abductive inference from partial observations, and constraint-aware reasoning involving constraint satisfaction or alternative discovery when constraints are violated~\cite{DBLP:journals/corr/abs-2206-14576}. These scenarios demand explicit structure reasoning mechanisms~\cite{DBLP:journals/fcsc/WangMFZYZCTCLZWW24}. 


\begin{example}
Consider an agent in Figure~\ref{fig:example}  tasked with helping Jack make a fruit salad. The agent knows that fruit salad requires fruits, a bowl, and a spoon, and that the steps are: chop, mix, and serve. From its memory, it knows that Jack has already chopped the fruits (ID 798), the bowl at home is broken (ID 2341), and a store downstairs has bowls, just 1 minute away (ID 5231). When asked, “What should Jack do next?”, a purely vector-based retrieval system can identify memory fragments mentioning “fruit salad” or “chopped fruits” based on semantic similarity. However, it cannot take into account that the bowl at home is broken or that a store nearby has bowls. As a result, it will only suggest mixing the fruits, ignoring the practical constraints. 
\end{example}


\begin{figure*}[t]
\centering
\includegraphics[width=0.95\textwidth]{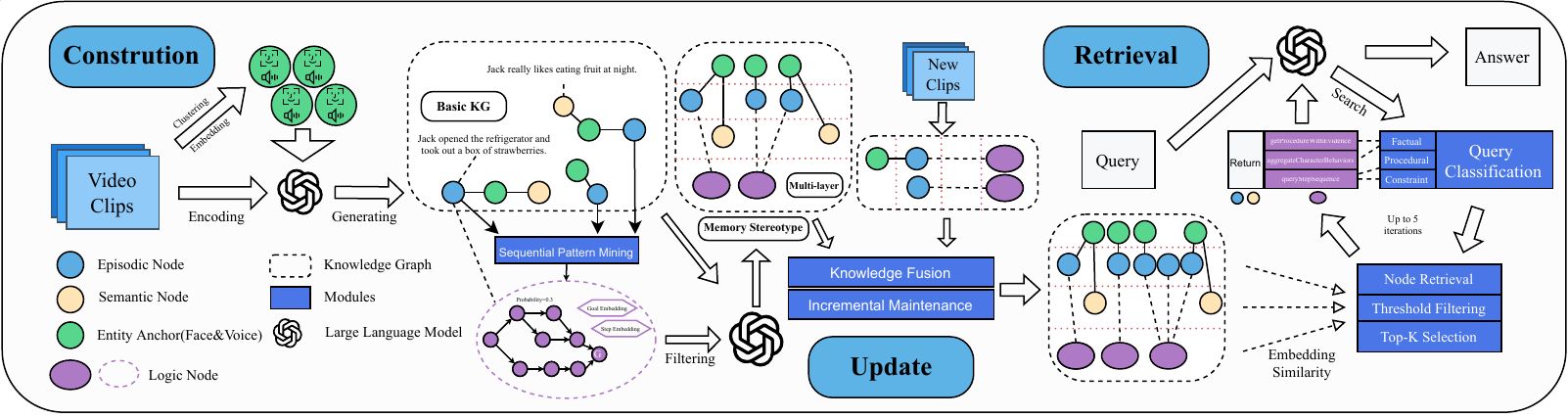}
\caption{Overview of the \ModelName{} framework. Raw multimodal data is processed through a three-layer memory prototype, maintained via the SK-Gen mechanism for distillation and incremental updates, and accessed through a hybrid retrieval framework designed for complex reasoning. }
\vspace{-4mm}
\label{fig:framework}
\end{figure*}

Neuro-symbolic AI~\cite{DBLP:series/faia/342, DBLP:journals/air/GarcezL23} aligns neural representations with System 1 intuitive pattern matching and symbolic rules with System 2 deterministic logic, offering a promising research direction. However, designing a memory system that supports this dual-process architecture poses three fundamental challenges:
(1) First, how to conceive a unified architecture in which neural representations and symbolic rules can coexist and interoperate effectively.
(2) Second, how to construct and continuously update such a memory system from large-scale multimodal data streams.
(3) Third, how to effectively retrieve and utilize both neural and symbolic memories to advance the reasoning capabilities of intelligent agents.

\noindent\textbf{Our approaches}. 
Guided by these challenges, we propose \textbf{\ModelName}, a neuro-symbolic memory system that combines neural representations and symbolic rules to bridge the gap between System 1 intuitive reasoning and System 2 deterministic reasoning. 
\ModelName{} provides a comprehensive framework comprising three integrated modules: (1) a unified prototype, (2) a scalable construction and update mechanism, and (3) a synergistic retrieval process.


First, we design a hierarchical framework that organizes information across three specialized layers: the episodic layer, the semantic layer, and the logic layer. The episodic layer records fine-grained multimodal observations with timestamps, while the semantic layer manages abstracted entity types and attributes. At the core of our system, the logic layer represents procedural knowledge through logic nodes. Each node couples a neural index with an explicit symbolic structure. The neural index is an aggregated neural embedding that enables discovery via vector similarity, and the symbolic structure is a procedural Directed Acyclic Graph (DAG) that encodes deterministic logical rules and step-wise dependencies.

Second, to construct and maintain the memory, we implement the \GenerationName{} mechanism, which automatically distills structured knowledge from accumulated multimodal experiences. It extracts temporally ordered action sequences and applies pattern mining to detect recurring procedural knowledge. Once a pattern is identified, the system constructs a symbolic DAG and computes the corresponding neural index. To handle continuous observations, the mechanism supports incremental updates, utilizing exponential moving averages to refine the neural index and structural modifications to update edges in the procedural DAGs and transition frequencies without full reconstruction.

Third, we develop a hybrid retrieval strategy that synergistically combines similarity-based search with deterministic symbolic query functions. The process begins by classifying incoming queries into factual, constraint, or procedural-based types to prioritize relevant memory layers. While neural retrieval identifies relevant nodes through embedding similarity, the system then executes symbolic functions directly on the retrieved structures. This allows the agent to generate precise, reproducible answers that satisfy explicit constraints, effectively merging intuitive semantic matching with rigorous, deterministic reasoning. 




\noindent\textbf{Contributions}. Our main contributions are as follows:
\begin{itemize}
    \item We propose \ModelName, a three-layer neuro-symbolic memory architecture that integrates an episodic layer, a semantic layer, and a neuro-symbolic layer to facilitate both neural discovery via memory prototypes and precise reasoning through symbolic structures with deterministic query functions.
    \item We develop \GenerationName{} to automate the extraction of structured knowledge from observations, which supports incremental updates to maintain consistency as new observations arrive.
    \item We introduce a hybrid retrieval mechanism that classifies queries by type and applies multi-level retrieval with symbolic enhancement for better reasoning.
    \item We demonstrate through experiments on real-world benchmarks that our approach achieves 4.35\% improvement in reasoning accuracy, with particularly strong gains on constraint-based queries.
\end{itemize}

    \section{Related Work}
    \label{sec.related}

\noindent
\textbf{Memory-Augmented Agents in Single Modality.}
Memory management has evolved to handle long-term context beyond fixed window constraints~\cite{peiyuan2024agile}. Common strategies involve storing entire trajectories, such as dialogues~\cite{lin2023mm,mei2024aios,wang2023enhancing,zhong2024memorybank} or execution traces~\cite{liu2024llm, hu2025hiagent,liu2024agentlite, sarch2023open, shang2024agentsquare}. Other methods utilize summaries~\cite{hu2025hiagent,wang2023enhancing,zhong2024memorybank} or latent embeddings~\cite{diko2025rewind, liu2024memlong, song2024moviechat, zhang2024flash} for persistence. Specialized architectures like MemGPT~\cite{packer2023memgpt} and Voyager~\cite{wang2023voyager} provide finer control over memory and skill acquisition. These single-modal systems focus on textual or symbolic data, forming the basis for more complex multimodal designs.

\noindent
\textbf{Multimodal Memory Systems.}
Recent work extends memory to multimodal environments, particularly for long-form video understanding~\cite{long2025seeing}. One approach uses pure neural representations, storing memories as latent embeddings. MA-LMM~\cite{he2024malmm}, MovieChat~\cite{song2024moviechat}, and Flash-VStream~\cite{zhang2024flash} employ memory mechanisms to compress video tokens or store encoded visual features~\cite{bao2024boosting, zhang2024internlm, he2024ma}. Another direction integrates relational graphs with neural embeddings. M3-Agent~\cite{m3agent} organizes memories into entity-centric VideoGraphs. Socratic Models~\cite{lin2023mm, zhang2024mm} use multimodal models to generate language-based descriptions. While effective for semantic matching via RAG~\cite{lewis2020retrieval}, these systems often lack the explicit symbolic structures needed for deterministic reasoning under complex constraints.

\noindent
\textbf{Neuro-Symbolic Integration.}
Neuro-symbolic AI combines neural perception with symbolic logic~\cite{garcez2020neurosymbolic, garcez2019neural}. Early models like Neuro-Symbolic VQA~\cite{yi2018neural} disentangle perception from reasoning by executing symbolic programs. Modern approaches like Program-aided Language Models~\cite{gao2023pal} and ViperGPT~\cite{suris2023vipergpt} use LLMs to generate executable code. Other methods embed symbolic knowledge into networks~\cite{rocktaschel2017end} or extract rules from representations~\cite{evans2018learning, yang2017differentiable}. However, symbolic execution is often treated as a one-off tool. Our work introduces a Neuro-Symbolic Layer where symbolic structures are persistently stored and updated alongside neural representations, merging retrieval efficiency with reasoning precision.

    \vspace{-2mm}
    \section{Problem Statement}
    \label{sec.preliminary}

An intelligent agent perceives its environment through multimodal observations, maintains a memory of past experiences, and leverages these memories to make informed decisions. In modern architectures, Large Language Models (LLMs) serve as the cognitive core, while a memory system $\mathcal{M}$ enables the agent to accumulate knowledge over extended time horizons. Such memory-augmented agents must store past experiences, react to new stimuli, and maintain continuity across long operational periods.

Specifically, the agent perceives a continuous stream of multimodal observations $\mathcal{O}=\{o_1, o_2, \ldots\}$, where each observation $o_t$ comprises visual frames $o_t^v$, audio signals $o_t^a$, and textual descriptions $o_t^s$ (e.g., transcribed speech or subtitles). Given a query $q \in \mathcal{Q}$, the agent must retrieve relevant information from $\mathcal{M}$ and generate an accurate answer $a \in \mathcal{A}$.

\noindent
\textbf{Problem Statement.}
Consider an agent that continuously receives a stream of video observations $\mathcal{O}$. The objective of a multimodal memory system is to maintain a structured memory $\mathcal{M}$ that captures long-term dependencies. For any query $q \in \mathcal{Q}$, the agent needs to dynamically retrieve relevant context from $\mathcal{M}$ to generate an accurate answer $a \in \mathcal{A}$, effectively supporting factual recall, procedural understanding and constraint-aware reasoning, within an online, memory-efficient framework.

    \vspace{-2mm}
    \section{Overall Framework of \ModelName{}}
    \label{sec.method}

\label{sec:method}



In this section, we present NS-Mem, a neuro-symbolic memory framework designed to unify probabilistic semantic matching with deterministic structural reasoning. The frequently used notations are summarized in Appendix~\ref{tab:symbol}. The overall time complexity analysis is summarized in ~\ref{sec:complexity}.

\subsection{Architecture of \ModelName{}}
\label{sec:architecture}

Effective reasoning in open-world environments demands a dual capability: the flexibility to recall concrete experiences (System 1) and the rigor to apply deterministic procedural rules (System 2). To unify these capabilities, we introduce a three-layer memory architecture designed to capture diverse forms of knowledge:

\begin{definition}

The memory system $\mathcal{M} = (\mathcal{L}_{epi}, \mathcal{L}_{sem}, \mathcal{L}_{logic}, \mathcal{E})$ consists of three layers: the episodic layer $\mathcal{L}_{epi}$, which stores timestamped observations; the semantic layer $\mathcal{L}_{sem}$, which maintains entity coherence; and the logic layer $\mathcal{L}_{logic}$, which encodes procedural rules.
The edges $\mathcal{E}$ define the relationships between these layers. Specifically, the logic layer $\mathcal{L}_{logic}$ has directed edges to both the episodic and semantic layers, denoted as $\mathcal{E}_{{logic} \to {epi}}$ and $\mathcal{E}_{{logic} \to {sem}}$, respectively. Moreover, the episodic and semantic layers are interconnected via shared entity anchors, represented by the edges $\mathcal{E}_{{epi} \leftrightarrow {sem}}$.


\end{definition}

Specifically, 
the proposed memory system consists of the Episodic Layer, 
Semantic Layer and Logic Layer.

\noindent\textbf{(1) Episodic Layer}. The episodic layer serves as the observational foundation of the memory system, recording fine-grained event descriptions grounded in multimodal perception. Each episodic node $e = (t, \mathbf{d}, \mathbf{v}_e)$ stores the timestamp $t$ indicating its temporal position within the observation stream, a textual description $\mathbf{d}$, and the corresponding embedding $\mathbf{v}_e = \phi(\mathbf{d}) \in \mathbb{R}^d$. The description $\mathbf{d}$ is an atomic event narrative that synthesizes visual and auditory signals into a unified textual representation.
Each description explicitly references recognized entities through perceptual entity anchors, which are persistent identity nodes established via clustering of face embeddings and voice embeddings (detailed in Section~\ref{sec:distillation}). These entity references create edges from episodic nodes to the corresponding anchors, enabling entity-centric indexing so that all events involving a given identity can be efficiently retrieved across the entire temporal span of the memory. 

\noindent\textbf{(2) Semantic Layer}. The semantic layer abstracts and consolidates knowledge at a higher level, maintaining entity-centric summaries that evolve as new observations accumulate. Each semantic node $s = (\text{type}, \text{attrs}, \mathbf{v}_s)$ encodes a specific facet of abstracted knowledge, where $\text{type}$ categorizes the knowledge modality, $\text{attrs}$ accumulates the descriptive content, and $\mathbf{v}_s = \phi(\text{attrs}) \in \mathbb{R}^d$ is the node embedding. 
Like episodic nodes, semantic nodes are connected to entity anchors through edges, but they differ fundamentally in their update semantics. Rather than appending every observation as a new node, the semantic layer employs a reinforcement-based consolidation policy to maintain knowledge coherence: when a new semantic observation $s_{new}$ arrives, the system computes its embedding similarity against existing semantic nodes that share at least one entity anchor; if the similarity exceeds a positive threshold $\tau_{pos}$, the existing node's confidence is reinforced by incrementing its associated edge weight, and no new node is created; only when no sufficiently similar node exists is a new semantic node inserted into the layer. 
\noindent\textbf{(3) Logic Layer}. Each Logic Node $\mathcal{N} = (id, c, \mathbf{I}, \mathcal{G}, \mathcal{F})$ pairs Index Vectors $\mathbf{I}$ for neural discovery with a Procedural DAG $\mathcal{G}$ for symbolic querying, along with a goal description $c$ and deterministic query functions $\mathcal{F}$. Each node also maintains \texttt{episodic\_links} $\subseteq \mathcal{L}_{epi}$ referencing supporting observations for evidence traceability.

\noindent\underline{Index Vectors.}
Since a procedure comprises multiple steps, user queries may match either the high-level goal or specific intermediate steps. To accommodate both granularities, we maintain dual-level Index Vectors $\mathbf{I} = \{\mathbf{i}_{goal}, \mathbf{i}_{step}\}$: the goal-level index $\mathbf{i}_{goal} = \phi(c)$ embeds the goal description of the procedure, while the step-level index $\mathbf{i}_{step} = \frac{1}{|S|}\sum_{s \in S} \phi(s)$ averages embeddings of all descriptions $S = \{s_1, \ldots, s_n\}$. This dual-index design mirrors how search engines index both document titles and contents, ensuring that both goal-oriented and step-specific queries can locate the relevant node.

\noindent\underline{Procedural DAG.}
Index Vectors solve the discovery problem but cannot answer structural questions such as step ordering or constraint satisfaction. For such queries, we represent explicit symbolic structure as a Procedural DAG $\mathcal{G} = (V, E, A)$ where $V = \{v_0, v_1, \ldots, v_n, v_{n+1}\}$ includes distinguished nodes $v_0 = \texttt{START}$ and $v_{n+1} = \texttt{GOAL}$, $E \subseteq V \times V$ encodes valid step transitions, and $A: V \rightarrow 2^{\text{Attr}}$ maps each node to relevant attributes.
DAGs offer three advantages: expressiveness through concurrent execution paths, constraint-aware filtering for alternatives, and probabilistic semantics via absorbing Markov chain modeling.
In our implementation, observations from individual videos initially produce single-path DAGs; through knowledge fusion (Section~\ref{sec:fusion}), these merge into multi-path DAGs capturing procedural variations.

\noindent\textbf{(4) Edges}. These layers play complementary roles and interact with each other to support complex reasoning and decision-making.
The Logic Layer connects to the Episodic Layer through \texttt{episodic\_links}: each Logic Node $\mathcal{N}$ maintains references to specific episodic observations that serve as evidence grounding for the abstracted procedure, enabling the system to trace back to the underlying observations when needed. In contrast, the Logic Layer relates to the Semantic Layer through conceptual extension: while the Semantic Layer stores static entity attributes, the Logic Layer captures dynamic behavioral patterns involving those entities. 

Within each layer, the Episodic and Semantic layers organize around entity anchors---perceptual nodes representing recognized identities from video observations. Episodic nodes and Semantic nodes both connect to relevant entity anchors, with temporal ordering represented implicitly through timestamps rather than explicit edges. The Logic Layer introduces Logic Nodes that are relatively independent of each other (representing distinct procedures) but connect downward to episodic evidence through \texttt{episodic\_links}.

\subsection{Memory Construction and Maintenance}
\label{sec:construction}

Next, we will introduce \GenerationName{} that automatically constructs and updates the memory. The process is summarized in Algorithm~\ref{alg:skgen}.




\subsubsection{Memory Construction Pipeline}
\label{sec:distillation}

As multimodal video streams arrive, the system segments them into clips and extracts perceptual features: face embeddings via ArcFace~\cite{deng2019arcface} and voice embeddings via ERes2Net~\cite{chen2023eres2net}, with detected instances clustered to establish persistent entity anchors. A vision-language model then processes each clip along with the detected face and voice features, generating two types of textual outputs: (1) atomic event descriptions capturing observable actions, dialogues, and scene details with entity references, and (2) high-level conclusions summarizing character attributes, interpersonal relationships, and contextual knowledge. 
The former become Episodic Nodes $e = (t, \mathbf{d}, \mathbf{v}_e)$; the latter populate Semantic Nodes $s = (\text{type}, \text{attrs}, \mathbf{v}_s)$  following the layer-specific update policies defined in Section~\ref{sec:architecture}. Algorithm~\ref{alg:skgen} formalizes the complete construction pipeline, including observation processing (Phase~1) and Logic Node distillation (Phase~2).

The pipeline then transforms episodic memories into Logic Nodes through five sequential steps.

\textbf{Step 1: Action Sequence Extraction.} From the observation stream $\mathcal{O} = \{o_1, o_2, \ldots, o_K\}$ and episodic memories $\mathcal{L}_{epi}$, we extract temporally-ordered action sequences. For each video or session $v$, we obtain $S_v = \textsc{ExtractActions}(\{e \in \mathcal{L}_{epi} : e.\text{video} = v\})$, producing the sequence set $S_{seq} = \{S_1, S_2, \ldots, S_V\}$ where each $S_v = [a_1, a_2, \ldots, a_L]$ is an ordered list of actions. Action extraction may use rule-based pattern matching on episodic descriptions or LLM-based conversion to structured action representations.

\textbf{Step 2: Sequential Pattern Mining.} We apply PrefixSpan~\cite{pei2001prefixspan}, a sequential pattern mining algorithm, to discover recurring procedural motifs. Unlike set-based mining algorithms, PrefixSpan preserves temporal ordering: the pattern $[\text{cut}, \text{blanch}]$ is distinct from $[\text{blanch}, \text{cut}]$. The algorithm efficiently explores the pattern space through projected databases, outputting all patterns $p$ satisfying $\text{support}(p) = |\{S \in S_{seq} : p \subseteq S\}| / |S_{seq}| \geq \sigma$ where $\sigma$ is the minimum support threshold. This yields candidate patterns $\mathcal{P}_{cand} = \{p : \text{support}(p) \geq \sigma\}$.

\textbf{Step 3: Knowledge Verification.} Frequent patterns are not necessarily meaningful procedures. For instance, $[\text{pick\_up}, \text{put\_down}]$ may be frequent but does not constitute coherent knowledge. We employ an LLM-based filter to evaluate whether each candidate represents complete, reusable knowledge: $score_p = \textsc{LLMVerify}(p, \mathcal{M}_{rel})$ where $\mathcal{M}_{rel}$ contains related memories as context. Only patterns with $score_p > \tau$ proceed to structure extraction.

\textbf{Step 4: DAG Construction.} For verified patterns, we construct the Procedural DAG by creating nodes $V = \{\texttt{START}\} \cup \{v_a : a \in p\} \cup \{\texttt{GOAL}\}$, edges $E = \{(v_i, v_{i+1}) : \text{consecutive actions in } p\}$, and extracting attributes $A(v)$ from associated episodic memories. We also establish \texttt{episodic\_links} pointing to the specific episodic nodes supporting this pattern.

\textbf{Step 5: Index Generation.} Finally, we compute the Index Vectors of the logic node by averaging all the vectors: $\mathbf{i}_{goal} = \phi(p.\text{goal})$ for goal-level matching and $\mathbf{i}_{step} = \frac{1}{|p.\text{steps}|}\sum_{s \in p.\text{steps}} \phi(s)$ for step-level matching. The complete Logic Node $\mathcal{N} = (id, c, \mathbf{I}, \mathcal{G}, \mathcal{F})$ is then added to $\mathcal{L}_{logic}$.


\subsubsection{Incremental Maintenance}
\label{sec:incremental}

When new observations arrive in an open-world setting, we avoid costly full reconstruction by incrementally updating affected Logic Nodes through coupled neural and symbolic refinement.

\noindent\textbf{Matching and Gating.} For each new observation $o_{new}$, we first identify the best-matching Logic Node via neural discovery: $\mathcal{N}^* = \arg\max_{\mathcal{N} \in \mathcal{L}_{logic}} \text{sim}(\phi(o_{new}), \mathbf{I}_{\mathcal{N}})$. Updates proceed only if the similarity exceeds a gating threshold $\delta$, preventing noise from corrupting established knowledge. Observations with similarity below $\delta$ may represent novel procedures; these accumulate in a candidate pool until sufficient evidence triggers a new distillation cycle.

\noindent\textbf{Neural Refinement via EMA.} As new observations accumulate, the semantic distribution of a procedure may drift---users may describe the same task with varying terminology, or the procedure itself may evolve. Static Index Vectors would become increasingly misaligned with current usage patterns, degrading retrieval accuracy. To maintain alignment while avoiding catastrophic forgetting of historical semantics, we update Index Vectors using Exponential Moving Average (EMA):
\begin{equation}
\mathbf{i}_{t+1} = \beta \cdot \mathbf{i}_t + (1-\beta) \cdot \phi(o_{new}), \quad \beta \in [0,1]
\end{equation}
where $\mathbf{i}_t$ is the current index vector and $\beta$ (default 0.9) controls the decay rate. EMA naturally balances stability with adaptability: high $\beta$ values yield stable indexes that resist noise, while lower values increase sensitivity to distributional shifts. Unlike direct replacement, which catastrophically overwrites historical information, EMA preserves established semantics while gradually incorporating new linguistic variations.

\noindent\textbf{Symbolic Refinement via Transition Statistics.} Beyond neural refinement, the symbolic structure also benefits from new observations. Real-world procedures exhibit variation: some steps are more commonly taken than others, and knowing these frequencies enables probabilistic reasoning about typical execution paths and alternative reliability. To capture this, we maintain edge-level transition counts $N_{ij}$ for each $(v_i, v_j) \in E$ in the Procedural DAG $\mathcal{G} = (V, E)$. Each observed transition $v_i \rightarrow v_j$ increments the count: $N_{ij} \leftarrow N_{ij} + 1$, yielding the estimated transition probability:
\begin{equation}
\hat{P}(v_j | v_i) = \frac{N_{ij}}{\sum_{k:(v_i, v_k) \in E} N_{ik}}
\end{equation}

When observations reveal previously unseen but valid action, we expand the graph by inserting new nodes or edges with initialized statistics, thereby increasing coverage of procedural diversity while preserving determinism for existing structure. 
A concern is whether these incrementally updated statistics actually converge to meaningful values, or might drift arbitrarily. Fortunately, the counting-based estimator enjoys strong theoretical guarantees:

\begin{theorem}[Posterior Consistency]
\label{thm:consistency}
As the number of observations $n \to \infty$, the estimated transition probabilities converge almost surely to the true underlying probabilities: $\hat{P}(v_j | v_i) \xrightarrow{a.s.} P^*(v_j | v_i)$.
\end{theorem}
\textit{Proof.} See Appendix~\ref{proof:consistency}.

\begin{algorithm}[t]
\caption{\GenerationName{}: Memory Construction and Maintenance}
\label{alg:skgen}
\footnotesize
\begin{algorithmic}[1]
\Require Observation stream $\mathcal{O}=\{o_1, o_2, \ldots, o_K\}$, consolidation thresholds $\tau_{pos}, \tau_{neg}$, support threshold $\sigma$, verification threshold $\tau$, gating threshold $\delta$, EMA coefficient $\beta$
\Ensure Memory system $\mathcal{M} = (\mathcal{L}_{epi}, \mathcal{L}_{sem}, \mathcal{L}_{logic})$

\State \textbf{// Phase 1: Observation Processing}
\State $\mathcal{A} \gets \emptyset$; $\mathcal{L}_{epi} \gets \emptyset$; $\mathcal{L}_{sem} \gets \emptyset$
\For{each clip $o_k$ in $\mathcal{O}$}
    \State $\mathbf{F}_k \gets \textsc{ArcFace}(o_k)$; \; $\mathbf{U}_k \gets \textsc{ERes2Net}(o_k)$ \Comment{Perceptual extraction}
    \State $\mathcal{A} \gets \textsc{ClusterAndTrack}(\mathcal{A}, \mathbf{F}_k, \mathbf{U}_k)$ \Comment{Entity anchor update}
    \State $D_k, C_k \gets \textsc{VLM}(o_k, \mathcal{A})$ \Comment{Descriptions \& conclusions}
    \For{each description $\mathbf{d} \in D_k$} \Comment{Episodic construction}
        \State $e \gets (t_k, \mathbf{d}, \phi(\mathbf{d}))$; \; $\mathcal{L}_{epi} \gets \mathcal{L}_{epi} \cup \{e\}$
        \State Link $e$ to each entity anchor $a \in \textsc{ParseEntities}(\mathbf{d})$
    \EndFor
    \For{each conclusion $s_{new} \in C_k$} \Comment{Semantic consolidation}
        \State $\mathcal{S}_{cand} \gets \{s \in \mathcal{L}_{sem} : \textsc{Entities}(s_{new}) \subseteq \textsc{Entities}(s)\}$
        \If{$\exists\, s \in \mathcal{S}_{cand}$: $\text{sim}(\phi(s_{new}), \mathbf{v}_s) > \tau_{pos}$}
            \State \textsc{Reinforce}$(s)$ \Comment{Edge weights $+1$}
        \ElsIf{$\exists\, s \in \mathcal{S}_{cand}$: $\text{sim}(\phi(s_{new}), \mathbf{v}_s) < \tau_{neg}$}
            \State \textsc{Weaken}$(s)$ \Comment{Edge weights $-1$; prune if $\leq 0$}
        \Else
            \State $\mathcal{L}_{sem} \gets \mathcal{L}_{sem} \cup \{(\text{type}_{new}, s_{new}, \phi(s_{new}))\}$
        \EndIf
    \EndFor
\EndFor

\State \textbf{// Phase 2: Logic Distillation}

\State $\mathcal{L}_{logic} \gets \emptyset$
\State $S_{seq} \gets \textsc{ExtractActionSequence}(\mathcal{O}, \mathcal{L}_{epi})$ \Comment{Step 1}
\State $\mathcal{P}_{cand} \gets \textsc{PrefixSpan}(S_{seq}, \sigma)$ \Comment{Step 2}
\For{each pattern $p \in \mathcal{P}_{cand}$}
    \State $\mathcal{M}_{rel} \gets \textsc{RetrieveRelatedMemories}(p, \mathcal{L}_{epi})$
    \State $score_p \gets \textsc{LLMVerify}(p, \mathcal{M}_{rel})$ \Comment{Step 3}
    \If{$score_p > \tau$}
        \State $\mathcal{G} \gets \textsc{ConstructDAG}(p, \mathcal{M}_{rel})$ \Comment{Step 4}
        \State $\mathbf{i}_{goal} \gets \phi(p.\text{goal})$; \quad $\mathbf{i}_{step} \gets \text{Mean}(\{\phi(s) : s \in p.\text{steps}\})$ \Comment{Step 5}
        \State $\mathcal{N} \gets (id, p.\text{goal}, \{\mathbf{i}_{goal}, \mathbf{i}_{step}\}, \mathcal{G}, \mathcal{F})$
        \State $\mathcal{L}_{logic} \gets \mathcal{L}_{logic} \cup \{\mathcal{N}\}$
    \EndIf
\EndFor

\State \textbf{// Phase 3: Incremental Maintenance}
\For{each new observation $o_{new}$}
    \State $\mathcal{N}^* \gets \arg\max_{\mathcal{N} \in \mathcal{L}_{logic}} \text{sim}(\phi(o_{new}), \mathbf{I}_{\mathcal{N}})$ \Comment{Matching}
    \If{$\text{sim}(\phi(o_{new}), \mathbf{I}_{\mathcal{N}^*}) > \delta$} \Comment{Gating}
        \State $\mathbf{i} \gets \beta \cdot \mathbf{i} + (1-\beta) \cdot \phi(o_{new})$ \quad for $\mathbf{i} \in \mathbf{I}_{\mathcal{N}^*}$ \Comment{Neural Refinement}
        \State $\mathcal{G}_{\mathcal{N}^*} \gets \textsc{UpdateTransitions}(\mathcal{G}_{\mathcal{N}^*}, o_{new})$ \Comment{Symbolic Refinement}
    \EndIf
\EndFor
\State \Return $\mathcal{M}$
\end{algorithmic}
\end{algorithm}

\subsubsection{Knowledge Fusion}
\label{sec:fusion}

Real-world procedures rarely have a single canonical execution. Different individuals may perform the same task with variations in step ordering, optional steps, or alternative methods. When the same procedure is observed across multiple videos, each observation initially yields a single-path DAG representing one execution variant. Maintaining these as separate Logic Nodes would fragment the knowledge base, making retrieval less effective and preventing the system from recognizing that these variants represent the same underlying procedure. 

We propose a Knowledge Fusion phase that merges single-path DAGs into a unified multi-path DAG through three operations: (1) node alignment via embedding similarity and optimal bipartite matching to identify semantically equivalent steps across DAGs; (2) edge union to preserve all observed transitions, creating branching points where procedures diverge; and (3) statistic pooling to combine transition counts via Bayesian conjugacy. The fused DAG captures the full space of procedural variations while maintaining accurate transition statistics, enabling constraint-based queries to explore all valid alternatives. 
The following theorem demonstrates that the fusion operation is sound, ensuring that it does not introduce spurious paths or corrupt statistics:

\begin{theorem}[Fusion Consistency]
\label{thm:fusion}
Assuming input DAGs are valid observations of the same underlying procedure, the fusion operation preserves correctness: aligned nodes correspond to the same action with high probability, pooled parameters equal the posterior from the union of observations, and the fused structure retains all valid alternatives.
\end{theorem}
\textit{Proof.} See Appendix~\ref{proof:fusion}.

\subsection{Hybrid Retrieval and Reasoning}
\label{sec:retrieval}

With the memory architecture established, the agent now needs to access the right knowledge at query time. This requires not only finding relevant memories across heterogeneous layers, but also extracting structured information from Logic Nodes when queries demand precise, constraint-aware answers.

\subsubsection{Query Classification}
\label{sec:query_class}

Different queries demand different retrieval strategies. We classify incoming queries $q \in \mathcal{Q}$ into three types $T \in \{\textsf{factual}, \textsf{constraint}, \textsf{character}\}$: Factual queries request event recall or entity attributes and are best served by $\mathcal{L}_{epi}$ and $\mathcal{L}_{sem}$. Constraint queries impose explicit feasibility constraints and require symbolic operations on $\mathcal{G}$ for resolution. Character queries seek personality traits, behavioral patterns, or role summaries of specific individuals and benefit from cross-procedure aggregation via the Logic Layer.
Classification employs a two-tier approach. A rule-based pre-filter provides fast initial classification based on lexical patterns indicating constraint or character-based intent; others default to factual. For ambiguous cases, an LLM-based classifier refines the prediction. The resulting classification $T$ guides subsequent retrieval weighting.

\subsubsection{Multi-Granularity Retrieval}
\label{sec:multi_retrieval}

A single query may relate to memory at different levels of abstraction: users sometimes ask about high-level goals and sometimes about specific intermediate steps. To handle both, retrieval proceeds in two stages that leverage the dual-level Index Vectors.

Stage I (Neural Discovery) performs broad similarity search across all memory layers. For Logic Nodes, retrieval scores combine goal-level and step-level matching:
\begin{equation}
\text{score}(q, \mathcal{N}) = \alpha \cdot \text{sim}(\phi(q), \mathbf{i}_{goal}) + (1-\alpha) \cdot \text{sim}(\phi(q), \mathbf{i}_{step})
\end{equation}
where $\alpha \in [0,1]$ (default 0.3) balances high-level intent matching against specific content matching. This dual-index approach ensures that both goal-oriented and step-specific queries can discover relevant Logic Nodes. The initial retrieval returns candidates $\mathcal{R}_{init}(q) = \{n \in \mathcal{M} : \text{score}(q, n) > \theta\}$ where $\theta$ is the threshold.

Stage II (Type-Aware Re-ranking) re-weights candidates based on the query classification $T$ to prioritize the most relevant layer:
\begin{equation}
\text{score}_{final}(n) = \text{score}_{init}(n) \cdot w_{\text{layer}}(n, T)
\end{equation}
where $w_{\text{layer}}$ assigns higher weights to Episodic/Semantic nodes for $T = \{\textsf{factual}\}$ and to Logic nodes for $T \in \{\textsf{constraint}, \textsf{character}\}$. This strategy ensures that constraint and character queries surface Logic Nodes while factual queries prioritize episodic evidence.

\subsubsection{Symbolic Enhancement for Reasoning}
\label{sec:symbolic}

Once a relevant Logic Node is retrieved, its Procedural DAG $\mathcal{G}$ provides structured knowledge that can be queried programmatically. This is where the symbolic component becomes essential: rather than asking an LLM to ``figure out'' step sequences or filter by constraints from unstructured text, we directly traverse the DAG to extract exactly the information needed---enumerating valid paths, filtering by attribute constraints, or aggregating cross-procedure statistics. These operations are fast ($O(|\Pi| \cdot L)$ for path enumeration) and deterministic, ensuring reproducible answers.

Formally, a symbolic query function is a mapping $f: (\mathcal{G}, x) \mapsto y$ where $x$ is query-specific input and $y$ is the structured output. We implement three core functions:

(1) \texttt{getProcedureWithEvidence}$(goal) \rightarrow (\mathcal{G}, \texttt{episodic\_links})$: Returns the Procedural DAG along with supporting episodic evidence for a specified goal. This function enables evidence-grounded reasoning by providing both the abstract procedure and the concrete observations from which it was derived.

(2) \texttt{queryStepSequence}$(goal, \mathcal{C}) \rightarrow \Pi_{\mathcal{C}}$: Returns all paths from \texttt{START} to \texttt{GOAL} satisfying constraints $\mathcal{C}$. Formally:
\begin{equation}
\Pi_{\mathcal{C}} = \{\pi \in \Pi(v_0, v_{n+1}) : \forall v \in \pi, A(v) \models \mathcal{C}\}
\end{equation}
where $\Pi(v_0, v_{n+1})$ denotes all paths in $\mathcal{G}$ and $A(v) \models \mathcal{C}$ indicates that node $v$'s attributes satisfy constraint $\mathcal{C}$. This function handles constraint queries by filtering paths whose every node fulfills the specified feasibility requirements.

(3) \texttt{aggregateCharacterBehaviors}$(person) \rightarrow \{\mathcal{N}_1, \ldots, \mathcal{N}_k\}$: Returns all Logic Nodes linked to a specified person entity, enabling character-centric aggregation queries. This cross-procedure aggregation cannot be answered by individual embeddings alone.

The following theorem demonstrates that the symbolic query functions are deterministic: 

\begin{theorem}[Determinism Guarantee]
\label{thm:determinism}
All symbolic query functions $f \in \mathcal{F}$ are deterministic: for any fixed $\mathcal{G}$ and input $x$, repeated invocations of $f(\mathcal{G}, x)$ always return identical output $y$.
\end{theorem}
\textit{Proof.} See Appendix~\ref{proof:determinism}.

\subsection{Case Study}
\label{sec:case_study}

\begin{figure}[t]
\centering
\includegraphics[width=\columnwidth]{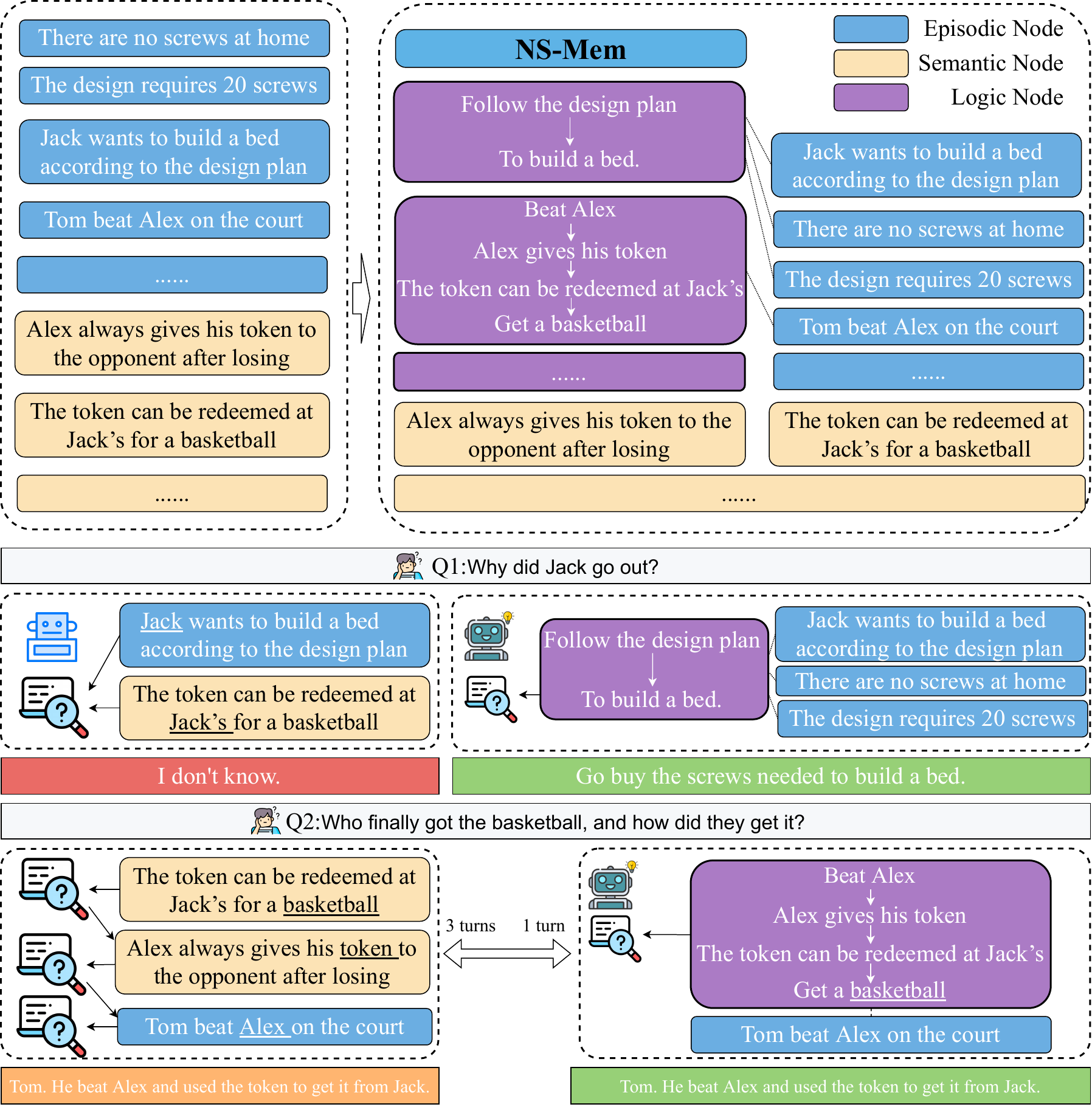}
\caption{Case study on vector-centric Memory and \ModelName{}.}
\label{fig:casestudy}
\end{figure}
Figure~\ref{fig:casestudy} demonstrates how NS-Mem outperforms vector-centric memory in reasoning tasks. In Q1, NS-Mem successfully infers Jack's intent by linking "building a bed" with "missing screws" via a Logic Node, whereas the baseline fails due to noise. In Q2, NS-Mem achieves the answer in a single turn by mapping the event "Tom beat Alex" to a pre-structured logic chain, while the baseline requires 3 turns of multi-hop retrieval. This highlights NS-Mem's superior capability in denoising and accelerating complex reasoning.

    \section{Experiments}
    \label{sec.experiment}
\label{sec.experiment}

\begin{table*}[t]
\centering
\caption{Performance comparison on M3-Bench-robot and M3-Bench-web. MD: Multi-Detail, MH: Multi-Hop, CM: Cross-Modal, HU: Human Understanding, GK: General Knowledge. Best results in each column are \underline{\textbf{underlined}}.}
\label{tab:main_results}
\small
\setlength{\tabcolsep}{4.2pt}
\begin{tabular*}{\textwidth}{@{\extracolsep{\fill}}lcccccccccccc@{}}
\toprule
& \multicolumn{6}{c}{\textbf{M3-Bench-robot}} & \multicolumn{6}{c}{\textbf{M3-Bench-web}} \\
\cmidrule(lr){2-7} \cmidrule(lr){8-13}
\textbf{Method} & \textbf{MD} & \textbf{MH} & \textbf{CM} & \textbf{HU} & \textbf{GK} & \textbf{All} & \textbf{MD} & \textbf{MH} & \textbf{CM} & \textbf{HU} & \textbf{GK} & \textbf{All} \\
\midrule
\multicolumn{13}{c}{\textit{Socratic Models}} \\
\cmidrule{1-13}
Qwen2.5-Omni-7b & 2.1 & 1.4 & 1.5 & 1.5 & 2.1 & 2.0 & 8.9 & 8.8 & 13.7 & 10.8 & 14.1 & 11.3 \\
Qwen2.5-VL-7b & 2.9 & 3.8 & 3.6 & 4.6 & 3.4 & 3.4 & 11.9 & 10.5 & 13.4 & 14.0 & 20.9 & 14.9 \\
Gemini-1.5-Pro & 6.5 & 7.5 & 8.0 & 9.7 & 7.6 & 8.0 & 18.0 & 17.9 & 23.8 & 23.1 & 28.7 & 23.2 \\
GPT-4o & 9.3 & 9.0 & 8.4 & 10.2 & 7.3 & 8.5 & 21.3 & 21.9 & 30.9 & 27.1 & 39.6 & 28.7 \\
\midrule
\multicolumn{13}{c}{\textit{Online Video Understanding}} \\
\cmidrule{1-13}
MovieChat & 13.3 & 9.8 & 12.2 & 15.7 & 7.0 & 11.2 & 12.2 & 6.6 & 12.5 & 17.4 & 11.1 & 12.6 \\
MA-LMM & 25.6 & 23.4 & 22.7 & 39.1 & 14.4 & 24.4 & 26.8 & 10.5 & 22.4 & 39.3 & 15.8 & 24.3 \\
Flash-VStream & 21.6 & 19.4 & 19.3 & 24.3 & 14.1 & 19.4 & 24.5 & 10.3 & 24.6 & 32.5 & 20.2 & 23.6 \\
\midrule
\multicolumn{13}{c}{\textit{Agent Methods}} \\
\cmidrule{1-13}
M3-Agent & 32.8 & 29.4 & 31.2 & 43.3 & 19.1 & 30.7 & 45.9 & 28.4 & 44.3 & 59.3 & 53.9 & 48.9 \\
\textbf{\ModelName{}} & \underline{\textbf{36.2}} & \underline{\textbf{31.5}} & \underline{\textbf{33.8}} & \underline{\textbf{45.7}} & \underline{\textbf{26.4}} & \underline{\textbf{34.7}} & \underline{\textbf{54.2}} & \underline{\textbf{34.6}} & \underline{\textbf{47.8}} & \underline{\textbf{60.1}} & \underline{\textbf{59.7}} & \underline{\textbf{53.6}} \\
\bottomrule
\end{tabular*}
\end{table*}

\begin{table}[t]
\centering
\vspace{-1em}
\caption{Accuracy by Query Type}
\label{tab:by_type}
\footnotesize
\begin{tabular}{lccc}
\toprule
\textbf{Method} & \textbf{Factual} & \textbf{Procedural} & \textbf{Constrained} \\
\midrule
M3-Agent & 52.5 & 23.8 & 25.0 \\
\ModelName{} & 54.3 & 35.7 & 37.5 \\
\midrule
$\Delta$ & +1.8 & +11.9 & +12.5 \\
\bottomrule
\end{tabular}
\end{table}

\begin{table}[t]
\centering
\caption{Efficiency Comparison on M3-Bench Datasets}
\label{tab:efficiency}
\footnotesize
\begin{tabular}{llccc}
\toprule
\textbf{Dataset} & \textbf{Metric} & \textbf{Baseline} & \textbf{\ModelName{}} & \textbf{$\Delta$} \\
\midrule
\multirow{2}{*}{Robot} & Avg. Rounds & 4.01 & 3.38 & -15.8\% \\
 & Avg. Time (sec) & 45.47 & 42.11 & -7.4\% \\
\midrule
\multirow{2}{*}{Web} & Avg. Rounds & 3.14 & 2.84 & -9.6\% \\
 & Avg. Time (sec) & 36.04 & 34.57 & -4.1\% \\
\bottomrule
\end{tabular}
\end{table}


\noindent\textbf{Datasets.} We evaluate our framework on M3-Bench~\cite{m3agent}, a comprehensive long-video question answering benchmark designed for memory-augmented agents. The benchmark consists of two primary subsets: \textbf{M3-Bench-robot}, which contains 100 real-world videos captured from a robot's perspective, and \textbf{M3-Bench-web}, which includes 920 web-sourced videos. The questions are categorized into five reasoning types: Multi-Detail (MD), Multi-Hop (MH), Cross-Modal (CM), Human Understanding (HU), and General Knowledge (GK). Due to computational constraints, We conduct evaluations on 50 videos (703 questions) for M3-Bench-robot and 550 videos (2,066 questions) for M3-Bench-web.

\noindent\textbf{Baseline Methods.} Following the previous work \cite{m3agent}, we compare \textbf{NS-Mem} against three categories of methods:
\begin{itemize}
\item\textbf{Socratic Models}. The methods directly query multimodal LLMs (Qwen2.5-Omni-7b, Qwen2.5-VL-7b, Gemini-1.5-Pro and GPT-4o) without explicit memory; 
\item \textbf{Online Video Understanding Methods}. The methods designed for streaming video processing (MovieChat\cite{song2024moviechat}, MA-LMM\cite{he2024malmm} and Flash-VStream\cite{zhang2024flash}); 
\item \textbf{Agent Method}. M3-Agent~\cite{m3agent} represents a state-of-the-art approach that utilizes episodic and semantic memory with vector-only retrieval.
\end{itemize}
\noindent\textbf{Metrics and Implementation.} Accuracy is evaluated using GPT-4o as the judge, following the standard M3-Bench protocol. For \ModelName{}, we set the hidden dimension to 512, the retrieval weight $\alpha$ to 0.3, and the verification threshold $\tau$ to 0.25. The incremental maintenance uses an EMA coefficient $\beta$ of 0.9. All experiments are conducted on a server with Intel(R) Xeon(R) Silver 4314 CPU, 512GB memory and NVIDIA RTX A5000 GPUs.

\begin{figure*}[t]
\centering
\begin{subfigure}{0.33\textwidth}
    \centering
    \includegraphics[width=\linewidth]{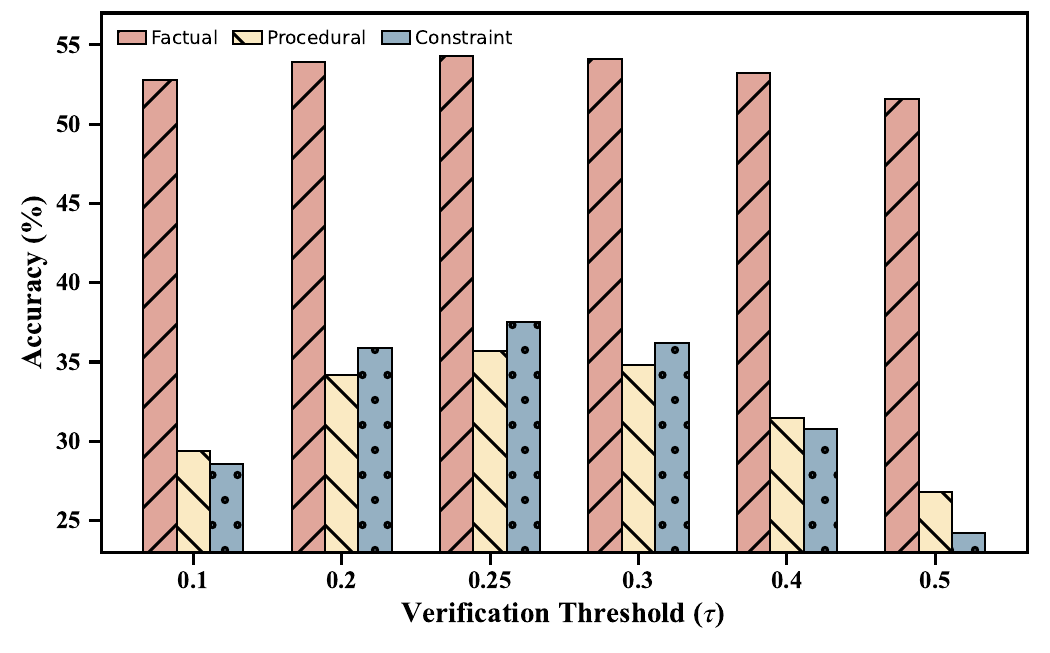}
    \caption{Impact of $\tau$}
    \label{fig:hyperparameter_tau}
\end{subfigure}
\begin{subfigure}{0.33\textwidth}
    \centering
    \includegraphics[width=\linewidth]{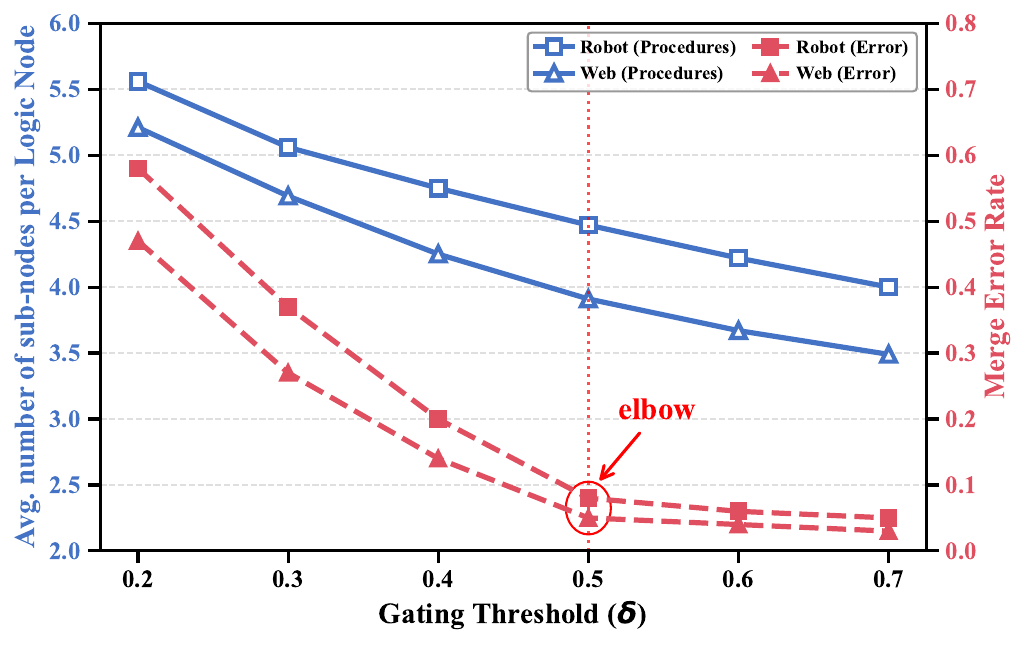}
    \caption{Impact of $\delta$}
    \label{fig:gating_threshold}
\end{subfigure}
\begin{subfigure}{0.33\textwidth}
    \centering
    \includegraphics[width=\linewidth]{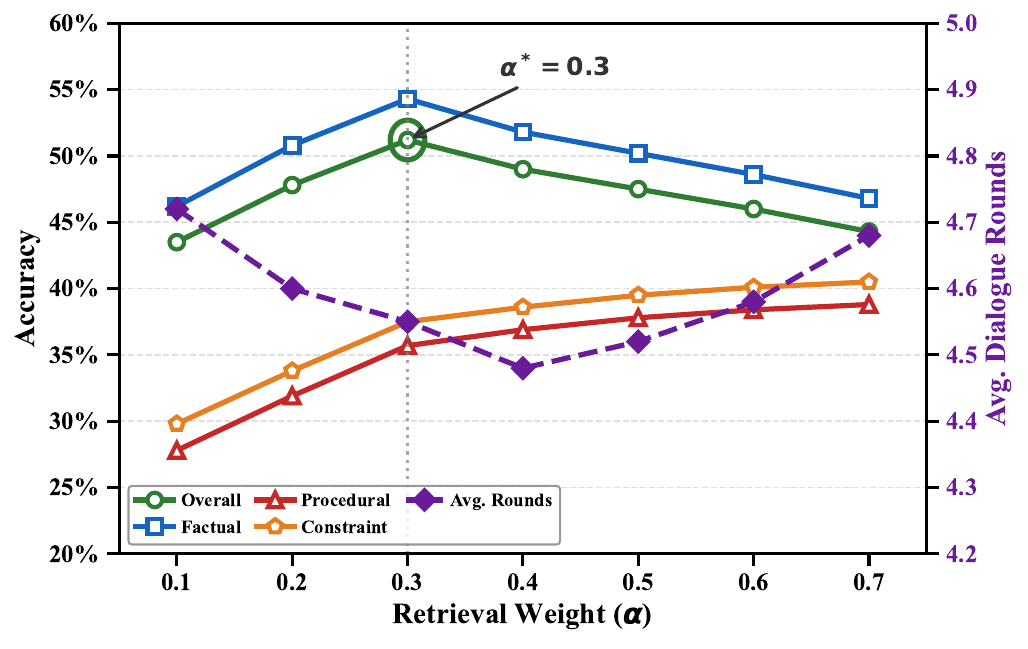}
    \caption{Impact of $\alpha$}
    \label{fig:alpha_sensitivity}
\end{subfigure}
\caption{Hyper-parameter analysis across different thresholds and weights. (a) Impact of $\tau$ on accuracy across query types. (b) Impact of of $\delta$ on knowledge consolidation and merge (c) Impact of $\alpha$ on accuracy and efficiency.}
\label{fig:hyperparameter_analysis}
\end{figure*}

\vspace{-3mm}
\subsection{Accuracy Comparison}
\noindent\textbf{Exp-1: Overall performance comparison.} In this experiment, we evaluate the overall accuracy of \ModelName{} compared to all baselines. The results are summarized in Table~\ref{tab:main_results}. As shown
in the table, \ModelName{} consistently outperforms all baseline methods across both Robot and Web datasets. Specifically, \ModelName{} achieves 53.6\% accuracy on M3-Bench-web and 34.7\% on M3-Bench-robot, representing absolute improvements of +4.7 and +4.0 points over M3-Agent. In contrast, Socratic Models and streaming methods show significantly lower performance. For instance, GPT-4o achieves only 8.5\% on Robot, which is 4.1$\times$ lower than \ModelName{}. This is because our neuro-symbolic architecture provides a structured substrate for reasoning about procedural sequences, which baselines fail to capture.

\noindent\textbf{Exp-2: Accuracy over different reasoning types.} We further analyze performance across the five reasoning types in Table~\ref{tab:main_results}. Notably, \ModelName{} demonstrates substantial gains in MH and GK. For MH, we observe relative gains of 21.8\% on Web, which is because the Procedural DAG enables explicit multi-step path enumeration. For GK, \ModelName{} shows a 38.2\% relative gain on Robot, benefiting from NS-Nodes that consolidate domain-specific procedural patterns from episodic observations. This validates that the logic layer effectively abstracts reusable knowledge from concrete experiences.

\noindent\textbf{Exp-3: Performance under different query types.} To understand where neuro-symbolic memory provides the most value, we break down the results by query type in Table~\ref{tab:by_type}. We observe that procedural and constrained queries benefit most from the neuro-symbolic layer, with relative improvements of 50.0\% for both. This is because the Procedural DAG explicitly encodes step-by-step logic, enabling deterministic constraint satisfaction via symbolic functions. In contrast, factual queries show 1.8\% improvement, as they primarily require direct recall rather than structured reasoning.

\subsection{Efficiency Evaluation}

\noindent\textbf{Exp-4: Efficiency.} In this experiment, we evaluate the efficiency of \ModelName{} in terms of retrieval rounds and time. The results are summarized in Table~\ref{tab:efficiency}. The results show that \ModelName{} significantly reduces the number of query rounds from 4.01 to 3.38 on Robot. This is because symbolic functions like \texttt{queryStepSequence()} can return complete procedural sequences in a single call, eliminating the iterative cycles required by pure neural memory. This reduction leads to concrete time savings of 7.4\% on Robot and 4.1\% on Web, even with the minimal overhead of symbolic execution.

\subsection{Ablation Study}

\noindent\textbf{Exp-5: Ablation study.} We evaluate the individual contributions of key components in Table~\ref{tab:ablation}. We can see that symbolic reasoning is the most critical component, providing 2.5$\times$ larger gains on Web compared to retrieval enhancement alone. Furthermore, we observe a synergistic interaction. For example, on dataset Web, the combined gain of the full model (+4.7) exceeds the sum of individual gains (+1.1 + +2.8 = 3.9). This is because the neural component (Index Vectors) improves retrieval precision, providing a more relevant substrate for symbolic reasoning to operate on.

\begin{table}[t]
\centering
\caption{Ablation Study on M3-Bench Datasets}
\label{tab:ablation}
\footnotesize
\begin{tabular}{llcc}
\toprule
\textbf{Configuration} & \textbf{Description} & \textbf{Web} & \textbf{Robot} \\
\midrule
Baseline & M3-Agent  & 48.9 & 30.7 \\
\midrule
w/o Symbolic & +Neuro +DAG  & 50.0 & 31.7 \\
w/o Neuro & +Symbolic +DAG & 51.7 & 33.1 \\
\midrule
\textbf{Full \ModelName{}} & \textbf{+Neuro +Symbolic +DAG} & \textbf{53.6} & \textbf{34.7} \\
\bottomrule
\end{tabular}
\end{table}

\subsection{Maintenance Evaluation}

\noindent\textbf{Exp-6: Incremental update performance.} In this experiment, we evaluate the capacity of \ModelName{} to maintain NS-Nodes as new observations arrive incrementally. We compare the performance of static graphs with dynamic graphs. The results are summarized in Table~\ref{tab:incremental}. The results show that dynamic graphs consistently outperform static graphs in both accuracy and efficiency. Specifically, for accuracy, the dynamic approach achieves 34.7\% on Robot and 53.6\% on Web, representing absolute improvements of +0.9\% and +0.6\% respectively. Regarding efficiency, the dynamic graph reduces the average number of rounds from 3.52 to 3.38 on Robot and from 2.92 to 2.84 on Web. This is because the EMA-based refinement effectively incorporates new procedural variations while preserving historical knowledge, preventing semantic staleness without requiring costly full reconstruction.

\begin{table}[t]
\centering
\caption{Incremental Update Evaluation.}
\label{tab:incremental}
\footnotesize
\begin{tabular}{llccc}
\toprule
\textbf{Metric} & \textbf{Dataset} & \textbf{Static} & \textbf{Dynamic} & \textbf{$\Delta$} \\
\midrule
\multirow{2}{*}{Accuracy (\%)} & Robot & 33.8 & 34.7 & +0.9 \\
 & Web & 53.0 & 53.6 & +0.6 \\
\midrule
\multirow{2}{*}{Avg. Rounds} & Robot & 3.52 & 3.38 & -0.14 \\
 & Web & 2.92 & 2.84 & -0.08 \\
\bottomrule
\end{tabular}
\end{table}

\subsection{Hyper-parameter Analysis}

\noindent\textbf{Exp-7: Verification threshold ($\tau$).} We analyze the impact of $\tau$ on accuracy in Figure~\ref{fig:hyperparameter_tau}. We observe that procedural and constraint queries are highly sensitive to $\tau$, peaking at $\tau=0.25$. This is because setting $\tau$ too low admits spurious patterns that break symbolic logic, while setting it too high overly filters valid procedural knowledge.

\noindent\textbf{Exp-8: Gating threshold ($\delta$).} As shown in Figure~\ref{fig:gating_threshold}, the gating threshold $\delta$ controls the balance between knowledge consolidation and noise prevention. The results show that merge error rates drop sharply until $\delta=0.5$, where they reach an elbow. This validates our choice of $\delta=0.5$ as the trade-off for incremental maintenance.

\noindent\textbf{Exp-9: Retrieval weight ($\alpha$).} We evaluate the impact of $\alpha$ in Figure~\ref{fig:alpha_sensitivity}. We observe that overall accuracy peaks at $\alpha=0.3$. This is because a moderate weight balances high-level goal intent with specific experiential grounding, ensuring that Memory Prototypes remain both discoverable and contextually accurate.

    \section{Conclusion}
    \label{sec.conclusion}
     In this paper, we presented \ModelName{}, a long-term neuro-symbolic memory framework to bridge the gap between intuitive neural retrieval and deterministic symbolic reasoning for multimodal agents. By integrating a hierarchical three-layer architecture with explicit logic rules and procedural DAGs, \ModelName{} enables agents to handle complex decision-making tasks that require constraint satisfaction and dependency reasoning. Our proposed \GenerationName{} mechanism further ensures that this memory can be automatically constructed and incrementally maintained from continuous multimodal observations. Extensive experiments on real-world benchmarks demonstrate that \ModelName{} significantly outperforms the state-of-the-art approach, particularly in constrained reasoning scenarios where symbolic structures provide rigorous logical grounding. 
    
    

    \newpage
    \normalem
    \bibliographystyle{ACM-Reference-Format}
    \bibliography{ref_new}  

@article{zhang2024flash,
  title={Flash-vstream: Memory-based real-time understanding for long video streams},
  author={Zhang, Haoji and Wang, Yiqin and Tang, Yansong and Liu, Yong and Feng, Jiashi and Dai, Jifeng and Jin, Xiaojie},
  journal={arXiv preprint arXiv:2406.08085},
  year={2024}
}

@book{russell2010artificial,
  title     = {Artificial Intelligence: A Modern Approach},
  author    = {Russell, Stuart and Norvig, Peter},
  year      = {2010},
  publisher = {Prentice Hall},
  edition   = {3rd}
}

@article{wooldridge1995intelligent,
  title   = {Intelligent agents: Theory and practice},
  author  = {Wooldridge, Michael and Jennings, Nicholas R},
  journal = {The Knowledge Engineering Review},
  volume  = {10},
  number  = {2},
  pages   = {115--152},
  year    = {1995}
}

@inproceedings{packer2023memgpt,
  title     = {MemGPT: Towards LLMs as Operating Systems},
  author    = {Packer, Charles and others},
  booktitle = {NeurIPS},
  year      = {2023}
}

@inproceedings{wang2023voyager,
  title     = {Voyager: An Open-Ended Embodied Agent with Large Language Models},
  author    = {Wang, Guanzhi and others},
  booktitle = {NeurIPS},
  year      = {2023}
}

@article{tulving1972episodic,
  title   = {Episodic and semantic memory},
  author  = {Tulving, Endel},
  journal = {Organization of memory},
  volume  = {1},
  pages   = {381--403},
  year    = {1972}
}

@article{DBLP:journals/air/GarcezL23,
  author       = {Artur d'Avila Garcez and
                  Lu{\'{\i}}s C. Lamb},
  title        = {Neurosymbolic {AI:} the 3rd wave},
  journal      = {Artif. Intell. Rev.},
  volume       = {56},
  number       = {11},
  pages        = {12387--12406},
  year         = {2023},
  url          = {https://doi.org/10.1007/s10462-023-10448-w},
  doi          = {10.1007/S10462-023-10448-W},
  bibsource    = {dblp computer science bibliography, https://dblp.org}
}

@book{DBLP:series/faia/342,
  editor       = {Pascal Hitzler and
                  Md. Kamruzzaman Sarker},
  title        = {Neuro-Symbolic Artificial Intelligence: The State of the Art},
  series       = {Frontiers in Artificial Intelligence and Applications},
  volume       = {342},
  publisher    = {{IOS} Press},
  year         = {2021},
  url          = {https://doi.org/10.3233/FAIA342},
  doi          = {10.3233/FAIA342},
  isbn         = {978-1-64368-244-0},
  bibsource    = {dblp computer science bibliography, https://dblp.org}
}

@article{DBLP:journals/corr/abs-2206-14576,
  author       = {Marcel Binz and
                  Eric Schulz},
  title        = {Using cognitive psychology to understand {GPT-3}},
  journal      = {CoRR},
  volume       = {abs/2206.14576},
  year         = {2022},
  url          = {https://doi.org/10.48550/arXiv.2206.14576},
  doi          = {10.48550/ARXIV.2206.14576},
  eprinttype    = {arXiv},
  eprint       = {2206.14576},
  bibsource    = {dblp computer science bibliography, https://dblp.org}
}

@inproceedings{DBLP:conf/iclr/Saparov023,
  author       = {Abulhair Saparov and
                  He He},
  title        = {Language Models Are Greedy Reasoners: {A} Systematic Formal Analysis
                  of Chain-of-Thought},
  booktitle    = {The Eleventh International Conference on Learning Representations,
                  {ICLR} 2023, Kigali, Rwanda, May 1-5, 2023},
  publisher    = {OpenReview.net},
  year         = {2023},
  url          = {https://openreview.net/forum?id=qFVVBzXxR2V},
  bibsource    = {dblp computer science bibliography, https://dblp.org}
}

@article{DBLP:journals/corr/abs-2305-15771,
  author       = {Karthik Valmeekam and
                  Matthew Marquez and
                  Sarath Sreedharan and
                  Subbarao Kambhampati},
  title        = {On the Planning Abilities of Large Language Models - {A} Critical
                  Investigation},
  journal      = {CoRR},
  volume       = {abs/2305.15771},
  year         = {2023},
  url          = {https://doi.org/10.48550/arXiv.2305.15771},
  doi          = {10.48550/ARXIV.2305.15771},
  eprinttype    = {arXiv},
  eprint       = {2305.15771},
  bibsource    = {dblp computer science bibliography, https://dblp.org}
}

@article{DBLP:journals/corr/abs-2212-09196,
  author       = {Taylor W. Webb and
                  Keith J. Holyoak and
                  Hongjing Lu},
  title        = {Emergent Analogical Reasoning in Large Language Models},
  journal      = {CoRR},
  volume       = {abs/2212.09196},
  year         = {2022},
  url          = {https://doi.org/10.48550/arXiv.2212.09196},
  doi          = {10.48550/ARXIV.2212.09196},
  eprinttype    = {arXiv},
  eprint       = {2212.09196},
  bibsource    = {dblp computer science bibliography, https://dblp.org}
}

@article{DBLP:journals/fcsc/WangMFZYZCTCLZWW24,
  author       = {Lei Wang and
                  Chen Ma and
                  Xueyang Feng and
                  Zeyu Zhang and
                  Hao Yang and
                  Jingsen Zhang and
                  Zhiyuan Chen and
                  Jiakai Tang and
                  Xu Chen and
                  Yankai Lin and
                  Wayne Xin Zhao and
                  Zhewei Wei and
                  Jirong Wen},
  title        = {A survey on large language model based autonomous agents},
  journal      = {Frontiers Comput. Sci.},
  volume       = {18},
  number       = {6},
  pages        = {186345},
  year         = {2024},
  url          = {https://doi.org/10.1007/s11704-024-40231-1},
  doi          = {10.1007/S11704-024-40231-1},
  bibsource    = {dblp computer science bibliography, https://dblp.org}
}

@inproceedings{lewis2020retrieval,
  title     = {Retrieval-Augmented Generation for Knowledge-Intensive NLP Tasks},
  author    = {Lewis, Patrick and others},
  booktitle = {NeurIPS},
  year      = {2020}
}

@article{garcez2020neurosymbolic,
  title   = {Neurosymbolic AI: The 3rd Wave},
  author  = {Garcez, Artur d'Avila and Lamb, Luis C},
  journal = {Artificial Intelligence Review},
  year    = {2023}
}

@inproceedings{rocktaschel2017end,
  title     = {End-to-End Differentiable Proving},
  author    = {Rocktäschel, Tim and Riedel, Sebastian},
  booktitle = {NeurIPS},
  year      = {2017}
}

@inproceedings{evans2018learning,
  title     = {Learning Explanatory Rules from Noisy Data},
  author    = {Evans, Richard and Grefenstette, Edward},
  booktitle = {JAIR},
  year      = {2018}
}

@inproceedings{suris2023vipergpt,
  title     = {ViperGPT: Visual Inference via Python Execution for Reasoning},
  author    = {Surís, Dídac and Menon, Sachit and Vondrick, Carl},
  booktitle = {ICCV},
  year      = {2023}
}

@inproceedings{song2024moviechat,
  title     = {MovieChat: From Dense Token to Sparse Memory for Long Video Understanding},
  author    = {Song, Enxin and others},
  booktitle = {CVPR},
  year      = {2024}
}

@inproceedings{he2024malmm,
  title     = {MA-LMM: Memory-Augmented Large Multimodal Model for Long-Term Video Understanding},
  author    = {He, Bo and others},
  booktitle = {CVPR},
  year      = {2024}
}

@article{m3agent,
  title={Seeing, listening, remembering, and reasoning: A multimodal agent with long-term memory},
  author={Long, Lin and He, Yichen and Ye, Wentao and Pan, Yiyuan and Lin, Yuan and Li, Hang and Zhao, Junbo and Li, Wei},
  journal={arXiv preprint arXiv:2508.09736},
  year={2025}
}

@inproceedings{wang2024videoagent,
  title     = {VideoAgent: Long-form Video Understanding with Large Language Model as Agent},
  author    = {Wang, Xiaohan and others},
  booktitle = {ECCV},
  year      = {2024}
}

@article{garcez2019neural,
  title   = {Neural-symbolic computing: An effective methodology for principled integration of machine learning and reasoning},
  author  = {Garcez, Artur d'Avila and Lamb, Luis C},
  journal = {Journal of Applied Logics},
  volume  = {6},
  number  = {4},
  pages   = {611--632},
  year    = {2019}
}

@inproceedings{yang2017differentiable,
  title     = {Differentiable Learning of Logical Rules for Knowledge Base Reasoning},
  author    = {Yang, Fan and others},
  booktitle = {NeurIPS},
  year      = {2017}
}

@inproceedings{gao2023pal,
  title     = {PAL: Program-aided Language Models},
  author    = {Gao, Luyu and others},
  booktitle = {ICML},
  year      = {2023}
}

@inproceedings{yi2018neural,
  title     = {Neural-Symbolic VQA: Disentangling Reasoning from Vision and Language Understanding},
  author    = {Yi, Kexin and others},
  booktitle = {NeurIPS},
  year      = {2018}
}

@inproceedings{tang2019coin,
  title     = {COIN: A Large-scale Dataset for Comprehensive Instructional Video Analysis},
  author    = {Tang, Yansong and others},
  booktitle = {CVPR},
  year      = {2019}
}

@inproceedings{pei2001prefixspan,
  title     = {PrefixSpan: Mining Sequential Patterns Efficiently by Prefix-Projected Pattern Growth},
  author    = {Pei, Jian and others},
  booktitle = {ICDE},
  year      = {2001}
}

@inproceedings{shridhar2020alfred,
  title     = {ALFRED: A Benchmark for Interpreting Grounded Instructions for Everyday Tasks},
  author    = {Shridhar, Mohit and others},
  booktitle = {CVPR},
  year      = {2020}
}

@article{long2025seeing,
  title={Seeing, listening, remembering, and reasoning: A multimodal agent with long-term memory},
  author={Long, Lin and He, Yichen and Ye, Wentao and Pan, Yiyuan and Lin, Yuan and Li, Hang and Zhao, Junbo and Li, Wei},
  journal={arXiv preprint arXiv:2508.09736},
  year={2025}
}

@article{peiyuan2024agile,
  title={Agile: A novel reinforcement learning framework of llm agents},
  author={Peiyuan, Feng and He, Yichen and Huang, Guanhua and Lin, Yuan and Zhang, Hanchong and Zhang, Yuchen and Li, Hang},
  journal={Advances in Neural Information Processing Systems},
  volume={37},
  pages={5244--5284},
  year={2024}
}

@article{lin2023mm,
  title={Mm-vid: Advancing video understanding with gpt-4v (ision)},
  author={Lin, Kevin and Ahmed, Faisal and Li, Linjie and Lin, Chung-Ching and Azarnasab, Ehsan and Yang, Zhengyuan and Wang, Jianfeng and Liang, Lin and Liu, Zicheng and Lu, Yumao and others},
  journal={arXiv preprint arXiv:2310.19773},
  year={2023}
}

@article{mei2024aios,
  title={Aios: Llm agent operating system},
  author={Mei, Kai and Zhu, Xi and Xu, Wujiang and Hua, Wenyue and Jin, Mingyu and Li, Zelong and Xu, Shuyuan and Ye, Ruosong and Ge, Yingqiang and Zhang, Yongfeng},
  journal={arXiv preprint arXiv:2403.16971},
  year={2024}
}

@article{wang2023enhancing,
  title={Enhancing large language model with self-controlled memory framework},
  author={Wang, Bing and Liang, Xinnian and Yang, Jian and Huang, Hui and Wu, Shuangzhi and Wu, Peihao and Lu, Lu and Ma, Zejun and Li, Zhoujun},
  journal={arXiv preprint arXiv:2304.13343},
  year={2023}
}

@inproceedings{zhong2024memorybank,
  title={Memorybank: Enhancing large language models with long-term memory},
  author={Zhong, Wanjun and Guo, Lianghong and Gao, Qiqi and Ye, He and Wang, Yanlin},
  booktitle={Proceedings of the AAAI Conference on Artificial Intelligence},
  volume={38},
  number={17},
  pages={19724--19731},
  year={2024}
}

@inproceedings{hu2025hiagent,
  title={Hiagent: Hierarchical working memory management for solving long-horizon agent tasks with large language model},
  author={Hu, Mengkang and Chen, Tianxing and Chen, Qiguang and Mu, Yao and Shao, Wenqi and Luo, Ping},
  booktitle={Proceedings of the 63rd Annual Meeting of the Association for Computational Linguistics (Volume 1: Long Papers)},
  pages={32779--32798},
  year={2025}
}

@article{liu2024llm,
  title={From llm to conversational agent: A memory enhanced architecture with fine-tuning of large language models},
  author={Liu, Na and Chen, Liangyu and Tian, Xiaoyu and Zou, Wei and Chen, Kaijiang and Cui, Ming},
  journal={arXiv preprint arXiv:2401.02777},
  year={2024}
}

@article{liu2024agentlite,
  title={Agentlite: A lightweight library for building and advancing task-oriented llm agent system},
  author={Liu, Zhiwei and Yao, Weiran and Zhang, Jianguo and Yang, Liangwei and Liu, Zuxin and Tan, Juntao and Choubey, Prafulla K and Lan, Tian and Wu, Jason and Wang, Huan and others},
  journal={arXiv preprint arXiv:2402.15538},
  year={2024}
}

@inproceedings{sarch2023open,
  title={Open-ended instructable embodied agents with memory-augmented large language models},
  author={Sarch, Gabriel and Wu, Yue and Tarr, Michael and Fragkiadaki, Katerina},
  booktitle={Findings of the Association for Computational Linguistics: EMNLP 2023},
  pages={3468--3500},
  year={2023}
}

@article{shang2024agentsquare,
  title={Agentsquare: Automatic llm agent search in modular design space},
  author={Shang, Yu and Li, Yu and Zhao, Keyu and Ma, Likai and Liu, Jiahe and Xu, Fengli and Li, Yong},
  journal={arXiv preprint arXiv:2410.06153},
  year={2024}
}

@inproceedings{diko2025rewind,
  title={ReWind: Understanding Long Videos with Instructed Learnable Memory},
  author={Diko, Anxhelo and Wang, Tinghuai and Swaileh, Wassim and Sun, Shiyan and Patras, Ioannis},
  booktitle={Proceedings of the Computer Vision and Pattern Recognition Conference},
  pages={13734--13743},
  year={2025}
}

@article{liu2024memlong,
  title={Memlong: Memory-augmented retrieval for long text modeling},
  author={Liu, Weijie and Tang, Zecheng and Li, Juntao and Chen, Kehai and Zhang, Min},
  journal={arXiv preprint arXiv:2408.16967},
  year={2024}
}

@article{bao2024boosting,
  title={Boosting micro-expression recognition via self-expression reconstruction and memory contrastive learning},
  author={Bao, Yongtang and Wu, Chenxi and Zhang, Peng and Shan, Caifeng and Qi, Yue and Ben, Xianye},
  journal={IEEE Transactions on Affective Computing},
  volume={15},
  number={4},
  pages={2083--2096},
  year={2024},
  publisher={IEEE}
}

@article{zhang2024internlm,
  title={Internlm-xcomposer2. 5-omnilive: A comprehensive multimodal system for long-term streaming video and audio interactions},
  author={Zhang, Pan and Dong, Xiaoyi and Cao, Yuhang and Zang, Yuhang and Qian, Rui and Wei, Xilin and Chen, Lin and Li, Yifei and Niu, Junbo and Ding, Shuangrui and others},
  journal={arXiv preprint arXiv:2412.09596},
  year={2024}
}

@inproceedings{he2024ma,
  title={Ma-lmm: Memory-augmented large multimodal model for long-term video understanding},
  author={He, Bo and Li, Hengduo and Jang, Young Kyun and Jia, Menglin and Cao, Xuefei and Shah, Ashish and Shrivastava, Abhinav and Lim, Ser-Nam},
  booktitle={Proceedings of the IEEE/CVF Conference on Computer Vision and Pattern Recognition},
  pages={13504--13514},
  year={2024}
}

@inproceedings{zhang2024mm,
  title={Mm-narrator: Narrating long-form videos with multimodal in-context learning},
  author={Zhang, Chaoyi and Lin, Kevin and Yang, Zhengyuan and Wang, Jianfeng and Li, Linjie and Lin, Chung-Ching and Liu, Zicheng and Wang, Lijuan},
  booktitle={Proceedings of the IEEE/CVF Conference on Computer Vision and Pattern Recognition},
  pages={13647--13657},
  year={2024}
}

@inproceedings{deng2019arcface,
  title={ArcFace: Additive Angular Margin Loss for Deep Face Recognition},
  author={Deng, Jiankang and Guo, Jia and Xue, Niannan and Zafeiriou, Stefanos},
  booktitle={Proceedings of the IEEE/CVF Conference on Computer Vision and Pattern Recognition},
  pages={4690--4699},
  year={2019}
}

@inproceedings{chen2023eres2net,
  title={An Enhanced Res2Net with Local and Global Feature Fusion for Speaker Verification},
  author={Chen, Yafeng and Zheng, Siqi and Wang, Hui and Cheng, Luyao and Chen, Qian and Qi, Jiajun},
  booktitle={Interspeech},
  pages={3032--3036},
  year={2023}
}

@book{kahneman2011thinking,
  title={Thinking, fast and slow},
  author={Kahneman, Daniel},
  year={2011},
  publisher={macmillan}
}

    \vspace{3pt}
    \newpage
    \appendix
    \section{Appendix}

\subsection{Notation Table}
\label{tab:symbol}
\begin{table}[t]
\centering
\caption{Key notations used throughout this paper.}
\begin{tabular}{|p{2.5cm}|p{5.7cm}|}
\hline
\textbf{Notation} & \textbf{Description} \\ \hline\hline
$\mathcal{M}$ & Memory system $\mathcal{M} = (\mathcal{L}_{epi}, \mathcal{L}_{sem}, \mathcal{L}_{logic}, \mathcal{E})$ \\ \hline
$\mathcal{L}_{epi}, \mathcal{L}_{sem}, \mathcal{L}_{logic}$ & Episodic, Semantic, and Logic layers \\ \hline
$\mathcal{O}=\{o_1, o_2, \ldots\}$ & Multimodal observation stream \\ \hline
$o_t = (o_t^v, o_t^a, o_t^s)$ & Observation at time $t$ (visual, audio, text) \\ \hline
$\mathcal{Q}, \mathcal{A}$ & Query space and answer space \\ \hline
$q \in \mathcal{Q}$ & A user query in natural language \\ \hline
$e = (t, \mathbf{d}, \mathbf{v}_e)$ & Episodic node: timestamp, content, episodic embedding \\ \hline
$s = (\text{type}, \text{attrs}, \mathbf{v}_s)$ & Semantic node: entity type, attributes, semantic embedding \\ \hline
$\mathcal{N} = (id, c, \mathbf{I}, \mathcal{G}, \mathcal{F})$ & Logic Node: id, goal description, index vectors, DAG, functions \\ \hline
$c$ & Natural language goal description 
\\ \hline
$\mathbf{I} = \{\mathbf{i}_{goal}, \mathbf{i}_{step}\}$ & Index Vectors: goal-level and step-level embeddings \\ \hline
$\mathbf{i}_{goal}, \mathbf{i}_{step}$ & Goal index $\phi(c)$ and step index $\frac{1}{|S|}\sum_{s}\phi(s)$ \\ \hline
$\mathcal{G} = (V, E, A)$ & Procedural DAG: nodes, edges, attribute function \\ \hline
$\mathcal{F}$ & Set of deterministic symbolic query functions \\ \hline
$\phi: \text{Text} \rightarrow \mathbb{R}^d$ & Text embedding function 
\\ \hline
$d$ & Embedding dimension \\ \hline
\end{tabular}
\end{table}

\subsection{Complexity Analysis}
\label{sec:complexity}

\ModelName{} operates continuously through two decoupled processes: incremental maintenance and query-driven reasoning. 
In terms of the incremental maintenance phase, for each incoming observation at time step $t$, the system performs pattern mining and prototype updates sequentially. 
The \GenerationName{} mechanism processes a sliding window of size $w$ to update the episodic buffer, with a time complexity of $O(w \times d)$ for embedding computation, where $d$ is the vector dimension.
Regarding the structural refinement, updating the transition statistics and edges in the Procedural DAG $\mathcal{G}=(V, E)$ takes $O(1)$ via hash-based lookups. 
The prototype update via EMA requires element-wise operations with a complexity of $O(d)$. 
Crucially, maintaining the vector index for $\mathcal{L}_{logic}$ involves incremental insertions. Using graph-based indexing structures, this requires $O(d \log |\mathcal{N}|)$, where $|\mathcal{N}|$ is the number of logic nodes.
Therefore, the total maintenance complexity per time step is dominated by the index update $O(d \log |\mathcal{N}|)$, which is highly efficient compared to batch retraining.

Regarding the retrieval and reasoning phase, the system first identifies relevant Logic Nodes $\mathcal{N}$ using the index-based retrieval.
For a specific query $q$, the search complexity is $O(d \log |\mathcal{N}|)$.
Once a relevant node is identified, symbolic reasoning is executed on the associated procedural DAG $\mathcal{G}$. 
The traversal of finding a path or checking constraints takes $O(|V| + |E|)$.

Thus, suppose that \ModelName{} runs for $T$ time steps and handles $Q$ queries, the overall time complexity is $O(T \times d \log |\mathcal{N}| + Q \times (d \log |\mathcal{N}| + |V| + |E|))$. 

\subsection{Proof of Theorem~\ref{thm:consistency} (Posterior Consistency)}
\label{proof:consistency}

\begin{proof}
We prove consistency for both parameter types.

\textbf{Node success rates (Beta-Binomial)}:
The posterior mean is:
\begin{equation}
\hat{R}(v) = \frac{\alpha + s}{\alpha + \beta + n}
\end{equation}
where $s$ is successes out of $n$ trials. As $n \to \infty$:
\begin{equation}
\hat{R}(v) = \frac{\alpha + s}{\alpha + \beta + n} = \frac{\alpha/n + s/n}{\alpha/n + \beta/n + 1} \to \frac{s}{n} \to R^*(v)
\end{equation}
by the strong law of large numbers ($s/n \xrightarrow{a.s.} R^*(v)$).

\textbf{Edge probabilities (Dirichlet-Multinomial)}:
The posterior mean for edge $(u, v_j)$ is:
\begin{equation}
\hat{p}_{u,v_j} = \frac{\gamma_j + c_j}{\sum_i (\gamma_i + c_i)}
\end{equation}
where $c_j$ is the count of transitions to $v_j$. As total observations $n = \sum c_i \to \infty$:
\begin{equation}
\hat{p}_{u,v_j} \to \frac{c_j}{n} \xrightarrow{a.s.} p^*_{u,v_j}
\end{equation}
again by the strong law of large numbers.

Both results follow from Doob's posterior consistency theorem for exponential families with compact parameter spaces.
\end{proof}

\subsection{Proof of Theorem~\ref{thm:fusion} (Fusion Consistency)}
\label{proof:fusion}

\begin{proof}
We show that the fusion algorithm preserves correctness under the assumption that input DAGs are valid observations of the same true procedure.

\textbf{Node alignment correctness}: Using semantic embeddings with cosine similarity and threshold $\tau = 0.8$, nodes representing the same action (with potentially different surface forms) are correctly matched with high probability. The Hungarian algorithm guarantees optimal bipartite matching.

\textbf{Parameter fusion correctness}: The Bayesian fusion rule:
\begin{align}
\alpha_{\text{fused}} &= \alpha_1 + \alpha_2 - 1 \\
\beta_{\text{fused}} &= \beta_1 + \beta_2 - 1
\end{align}
is equivalent to pooling observations: if DAG 1 observed $(s_1, n_1)$ successes/trials and DAG 2 observed $(s_2, n_2)$, the fused posterior is:
\begin{equation}
\text{Beta}(\alpha_0 + s_1 + s_2, \beta_0 + (n_1 - s_1) + (n_2 - s_2))
\end{equation}
which is the correct posterior for combined observations.

\textbf{Structure preservation}: New edges (alternative paths) discovered in one video but not another are added to the fused DAG with appropriate Laplace-smoothed initial probabilities. This ensures no valid alternatives are lost.

By Theorem~\ref{thm:consistency}, as more videos are fused, parameters converge to true values, and the structure asymptotically captures all valid paths.
\end{proof}








\subsection{Proof of Theorem~\ref{thm:determinism} (Determinism)}
\label{proof:determinism}


\begin{proof}
We prove that each of the three symbolic query functions produces deterministic outputs.

\textbf{Function 1} (\texttt{getProcedureWithEvidence}): Given a goal description, this function retrieves the corresponding Logic Node and returns its Procedural DAG $\mathcal{G}$ together with the associated \texttt{episodic\_links}. Since the Logic Node is identified through a fixed similarity computation on static Index Vectors, and both $\mathcal{G}$ and \texttt{episodic\_links} are stored data structures, the output is a deterministic lookup with no stochastic component.

\textbf{Function 2} (\texttt{queryStepSequence}): Given a goal and constraint set $\mathcal{C}$, this function enumerates all paths $\Pi(v_0, v_{n+1})$ from \texttt{START} to \texttt{GOAL} in the fixed DAG $\mathcal{G}$ via depth-first traversal, then filters by checking $A(v) \models \mathcal{C}$ for every node $v$ along each path. Graph traversal on a static structure is deterministic, and attribute-constraint satisfaction is a Boolean predicate evaluated through set/arithmetic operations. Hence the filtered path set $\Pi_{\mathcal{C}}$ is uniquely determined by $(\mathcal{G}, \mathcal{C})$.

\textbf{Function 3} (\texttt{aggregateCharacterBehaviors}): Given a person entity identifier, this function scans the Logic Layer $\mathcal{L}_{logic}$ and collects all Logic Nodes whose \texttt{episodic\_links} reference episodic nodes associated with the specified entity anchor. Both the entity-anchor association and the \texttt{episodic\_links} are fixed stored references, making the aggregation a deterministic filtering operation over a static set.

Since none of the three functions involves random sampling, stochastic models, or LLM inference during computation, all outputs are reproducible given identical inputs.
\end{proof}





    \label{sec.appendix}
    \balance
\end{document}